\documentclass[a4paper]{article}

\usepackage[a4paper,top=3cm,bottom=2cm,left=3cm,right=3cm,marginparwidth=1.75cm]{geometry}

\usepackage[english]{babel}
\usepackage[utf8x]{inputenc}
\usepackage[T1]{fontenc}
\usepackage{bbm}
\usepackage{amsmath}
\usepackage{amsthm}
\usepackage{natbib}

\DeclareMathOperator*{\argmin}{arg\,min}

\usepackage{graphicx}
\usepackage[colorinlistoftodos]{todonotes}
\usepackage[colorlinks=true, allcolors=blue]{hyperref}
\newtheorem{theorem}{Theorem}
\newtheorem{lemma}[theorem]{Lemma}
\usepackage{caption}
\usepackage{subcaption}
\usepackage{multirow}

\newcommand{\new}{\textcolor{black}}
\newcommand{\zzz}{\textcolor{black}}

\title{Unbiased Measurement of Feature Importance in Tree-Based Methods}

\author{Zhengze Zhou\footnote{PhD Candidate, Department of Statistical Science, Cornell University, zz433@cornell.edu.},  Giles Hooker\footnote{Associate Professor, Department of Statistical Science, Cornell University, gjh27@cornell.edu.} \\
}

\begin{document}
\maketitle

\begin{abstract}

We propose a modification that corrects for split-improvement variable importance measures in Random Forests and other tree-based methods. These methods have been shown to be biased towards increasing the importance of features with more potential splits. We show that by appropriately incorporating split-improvement as measured on out of sample data, this bias can be corrected yielding better summaries and screening tools.
\end{abstract}

\section{Introduction}

This paper examines split-improvement feature importance scores for tree-based methods.  Starting with Classification and Regression Trees \citep[CART;][]{breiman1984classification} and C4.5 \citep{quinlan2014c4}, decision trees have been a workhorse of general machine learning, particularly within ensemble methods such as Random Forests \citep[RF;][]{breiman2001random} and Gradient Boosting Trees \citep{friedman2001greedy}. They enjoy the benefits of computational speed, few tuning parameters and natural ways of handling missing values. Recent statistical theory for ensemble methods \citep[e.g.][]{denil2014narrowing,scornet2015consistency,mentch2016quantifying,wager2018estimation,zhou2018boulevard} has provided theoretical guarantees and allowed formal statistical inference. \zzz{Variants of these models have also been proposed such as Bernoulli Random Forests \citep{yisen2016bernoulli,wang2017novel} and Random Survival Forests \citep{ishwaran2008random}.} For all these reasons, tree-based methods have seen broad applications including in protein interaction models \citep{meyer2017interactome} in product suggestions on Amazon \citep{sorokina2016amazon} and in financial risk management \citep{khaidem2016predicting}.

However, in common with other machine learning models, large ensembles of trees act as ``black boxes'', providing predictions but little insight as to how they were arrived at. There has thus been considerable interest in providing tools either to explain the broad patterns that are modeled by these methods, or to provide justifications for particular predictions. This paper examines variable or feature\footnote{We use ``feature'', ``variable'' and ``covariate'' interchangeably here to indicate individual measurements that act as inputs to a machine learning model from which a prediction is made.} importance scores that provide global summaries of how influential a particular input dimension is in the models' predictions. These have been among the earliest diagnostic tools for machine learning and have been put to practical use  as screening tools, see for example \citet{diaz2006gene} and \citet{menze2009comparison}. Thus, it is crucial that these feature importance measures reliably produce well-understood summaries.

Feature importance scores for tree-based models can be broadly split into two categories. Permutation methods rely on measuring the change in value or accuracy when the values of one feature are replaced by uninformative noise, often generated by a permutation. These have the advantage of being applicable to any function, but have been critiqued by \citet{hooker2007generalized,strobl2008conditional,hooker2019please} for forcing the model to extrapolate. By contrast, in this paper we study the alternative split-improvement scores \zzz{(also known as Gini importance, or mean decrease impurity)} that are specific to tree-based methods. These naturally aggregate the improvement associated with each note split and can be readily recorded within the tree building process \citep{breiman1984classification,friedman2001greedy}. In Python, split-improvement is the default implementation for almost every tree-based model, including \textbf{RandomForestClassifier}, \textbf{RandomForestRegressor}, \textbf{GradientBoostingClassifier} and \textbf{GradientBoostingRegressor} from \textbf{scikit-learn} \citep{pedregosa2011scikit}.

Despite their common use, split-improvement measures are biased towards features that exhibit more potential splits and in particular towards continuous features or features with large numbers of categories. \zzz{This weakness was already noticed in \citet{breiman1984classification} and \citet{strobl2007bias} conducted thorough experiments followed by more discussions in \citet{boulesteix2011random} and \citet{nicodemus2011letter}\footnote{See \underline{https://explained.ai/rf-importance/} for a popular demonstration of this.}.} While this may not be concerning when all covariates are similarly configured, in practice it is common to have a combination of categorical and continuous variables in which emphasizing more complex features may mislead any subsequent analysis. For example, gender will be a very important binary predictor in applications related to medical treatment; whether the user is a paid subscriber is also central to some tasks such as in Amazon and Netflix. But each of these may be rated as less relevant to age \zzz{which is a more complex feature} in either case. \zzz{In the task of ranking single nucleotide polymorphisms with respect to their ability to predict a target phenotype, researchers may overlook rare variants as common ones are systematically favoured by the split-improvement measurement. \citep{boulesteix2011random}}.

We offer an intuitive rationale for this phenomenon and design a simple fix to solve the bias problem. The observed bias is similar to overfitting in training machine learning models, where we should not build the model and evaluate relevant performance using the same set of data. To fix this, split-improvement calculated from a separate test set is taken into consideration. We further demonstrate that this new measurement is unbiased in the sense that features with no predictive power for the target variable will receive an importance score of zero in expectation. \new{These measures can be very readily implemented in tree-based software packages.} We believe the proposed measurement provides a more sensible means for evaluating feature importance in practice.

In the following, we introduce some background and notation for tree-based methods in Section \ref{tree}. In Section \ref{FI}, split-improvement is described in detail and its bias and limitations are presented.  The proposed unbiased measurement is introduced in Section \ref{USI}. Section \ref{real} applies our idea to a simulated example and three real world data sets. We conclude with some discussions and future directions in Section \ref{dis}. Proofs and some additional simulation results are collected in Appendix \ref{proof} and \ref{sim} respectively. 

\section{Tree-Based Methods}\label{tree}

In this section, we provide a brief introduction and mathematical formulation of tree-based models \new{that will also serve to introduce our notation}. We refer readers to relevant chapters in \citet{friedman2001elements} for a more detailed presentation. 

\subsection{Tree Building Process}\label{cart}

Decision trees are a non-parametric machine learning tool for constructing prediction models from data. They are obtained by recursively partitioning feature space by axis-aligned splits and fitting a simple prediction function, usually constant, within each partition. The result of this partitioning procedure is represented as a binary tree. Popular tree building algorithms, such as CART and C4.5, may differ in how they choose splits or deal with categorical features. Our introduction in this section mainly reflects how decision trees are implemented in \textbf{scikit-learn}.

Suppose our data consists of $p$ inputs and a response, denoted by $z_i = (x_i, y_i)$ for $i = 1, 2, \ldots, n$, with $x_i = (x_{i1}, x_{i2}, \ldots, x_{ip})$. For simplicity we assume our inputs are continuous\footnote{Libraries in different programming languages differ on how to handle categorical inputs. \textbf{rpart} and \textbf{randomForest} libraries in \textbf{R} search over every possible subsets when dealing with categorical features. However, tree-based models in \textbf{scikit-learn} do not support categorical inputs directly. Manually transformation is required to convert categorical features to integer-valued ones, such as using dummy variables, or treated as ordinal when applicable.}. Labels can be either continuous (regression trees) or categorical (classification trees). Let the data at a node $m$ represented by $Q$. Consider a splitting variable $j$ and a splitting point $s$, which results in two child nodes: 
$$
Q_{l} = \{(x, y)|x_j \leq s \}
$$
$$
Q_{r} = \{(x, y)|x_j > s \}.
$$
The impurity at node $m$ is computed by a function $H$, which acts as a measure for goodness-of-fit \new{and is invariant to sample size}. Our loss function for split $\theta = (j, s)$ is defined as the weighted average of the impurity at two child nodes:
$$
L(Q, \theta) = \frac{n_{l}}{n_m}H(Q_{l}) + \frac{n_{r}}{n_m}H(Q_{r}),
$$
where $n_m, n_l, n_r$ are the number of training examples falling into node $m, l, r$ respectively.
The best split is chosen by minimizing the above loss function:
\begin{equation}\label{best}
    \theta^* = \argmin_{\theta}L(Q, \theta).
\end{equation}
The tree is built by recursively splitting child nodes until some stopping criterion is met. For example, we may want to limit tree depth, or keep the number of training samples above some threshold within each node.

For regression trees, $H$ is usually chosen to be mean squared error, using average values as predictions within each node. At node $m$ with $n_m$ observations, $H(m)$ is defined as:
$$
\bar{y}_m = \frac{1}{n_m}\sum_{x_i \in m}y_i,
$$
$$
H(m) = \frac{1}{n_m}\sum_{x_i \in m}(y_i - \bar{y}_m)^2.
$$
Mean absolute error can also be used depending on specific application.

In classification, there are several different choices for the impurity function $H$. Suppose for node $m$, the target $y$ can take values of $1, 2, \ldots, K$, define
$$
p_{mk} = \frac{1}{n_m}\sum_{x_i \in m}\mathbbm{1}(y_i = k)
$$
to be the proportion of class $k$ in node $m$, for $k = 1, 2, \ldots, K$. Common choices are:
\begin{enumerate}
\item Misclassification error:
$$
H(m) = 1 - \max_{1 \leq k \leq K} p_{mk}.
$$
\item Gini index:
$$
H(m) = \sum_{k \neq k'}p_{mk}p_{mk'} = 1 - \sum_{k = 1}^K p_{mk}^2.
$$
\item Cross-entropy or deviance:
$$
H(m) = -\sum_{k=1}^K p_{mk}\log p_{mk}.
$$
\end{enumerate}

This paper will focus on mean squared error for regression and Gini index for classification. 

\subsection{Random Forests and Gradient Boosting Trees}

Though intuitive and interpretable, there are two major drawbacks associated with a single decision tree: they suffer from high variance and in some situations they are too simple to capture complex signals in the data. Bagging \citep{breiman1996bagging} and boosting \citep{friedman2001greedy} are two popular techniques used to improve the performance of decision trees.

Suppose we use a decision tree as a base learner $t(x; z_1, z_2, \ldots, z_n)$, where $x$ is the input for prediction and $z_1, z_2, \ldots, z_n$ are training examples as before. Bagging aims to stabilize the base learner $t$ by resampling the training data. In particular, the bagged estimator can be expressed as:
$$
\hat{t}(x) = \frac{1}{B}\sum_{b=1}^Bt(x; z^*_{b1}, z^*_{b2}, \ldots, z^*_{bn})
$$
where $z^*_{bi}$ are drawn independently with replacement from the original data (bootstrap sample), and $B$ is the total number of base learners. Each tree is constructed using a different bootstrap sample from the original data. Thus approximately one-third of the cases are left out and not used in the construction of each base learner. We call these \textit{out-of-bag} samples.

Random Forests \citep{breiman2001random} are a popular extension of bagging with an additional randomness injected. At each step when searching for the best split, only $p_0$ features are randomly selected from all $p$ possible features and the best split $\theta^*$ must be chosen from this subset. When $p_0 = p$, this reduces to bagging. Mathematically, the prediction is written as 
$$
\hat{t}^{RF}(x) = \frac{1}{B}\sum_{b=1}^Bt(x; \xi_b, z^*_{b1}, z^*_{b2}, \ldots, z^*_{bn})
$$
with $\xi_b \stackrel{\mbox{iid}}{\sim} \Xi$ denoting the additional randomness for selecting from a random subset of available features.

Boosting is another widely used technique by data scientists to achieve state-of-the-art results on many machine learning challenges \citep{chen2016xgboost}. Instead of building trees in parallel as in bagging, it does this sequentially, allowing the current base learner to correct for any previous bias. In \citet{ghosal2018boosting}, the authors also consider boosting RF to reduce bias. We will skip over some technical details on boosting and restrict our discussion of feature importance in the context of decision trees and RF. Note that as long as tree-based models combine base learners in an additive fashion, their feature importance measures are naturally calculated by (weighted) average across those of individual trees.

\section{Measurement of Feature Importance}\label{FI}

Almost every feature importance measures used in tree-based models belong to two classes: split-improvement or permutation importance. Though our focus will be on split-improvement, permutation importance is introduced first for completeness. 

\subsection{Permutation Importance}

Arguably permutation might be the most popular method for assessing feature importance in the machine learning community. Intuitively, if we break the link between a variable $X_j$ and $y$, the prediction error increases then variable $j$ can be considered as important. 

Formally, we view the training set as a matrix $X$ of size $n \times p$, where each row $x_i$ is one observation. Let $X^{\pi, j}$ be a matrix achieved by permuting the $j^{th}$ column according to some mechanism $\pi$. If we use $l(y_i, f(x_i))$ as the loss incurred when predicting $f(x_i)$ for $y_i$, then the importance of $j^{th}$ feature is defined as:
\begin{equation} \label{p}
\text{VI}^{\pi}_j = \sum_{i=1}^n l(y_i, f(x_i^{\pi, j}) - l(y_i, f(x_i)))
\end{equation}
the increase in prediction error when the $j^{th}$ feature is permuted. Variations include choosing different permutation mechanism $\pi$ or evaluating Equation (\ref{p}) on a separate test set. In Random Forests, \citet{breiman2001random} suggest to only permute the values of the $j^{th}$ variable in the \textit{out-of-bag} samples for each tree, and final importance for the forest is given by averaging across all trees. 

There is a small literature analyzing permutation importance in the context of RF. \citet{ishwaran2007variable} studied paired importance. 
\citet{hooker2007generalized, strobl2008conditional, hooker2019please} advocated against permuting features by arguing it emphasizes behavior in regions where there is very little data. More recently, \citet{gregorutti2017correlation} conducted a theoretical analysis of permutation importance measure for an additive regression model. 

\subsection{Split-Improvement}

While permutation importance measures can generically be applied to any prediction function, 
split-improvement is unique to tree-based methods, and can be calculated directly from the training process. Every time a node is split on variable $j$,  the combined impurity for the two descendent nodes is less than the parent node. Adding up the weighted impurity decreases for each split in a tree and averaging over all trees in the forest yields an importance score for each feature.

Following our notation in Section \ref{cart}, the impurity function $H$ is either mean squared error for regression or Gini index for classification. The best split at node $m$ is given by $\theta^*_m$ which splits at $j^{th}$ variable and results in two child nodes denoted as $l$ and $r$. Then the decrease in impurity for split $\theta^*$ is defined as:
\begin{equation} \label{di}
\Delta(\theta^*_m) =  \omega_mH(m) - (\omega_lH(l) + \omega_rH(r)),
\end{equation}
where $\omega$ is the proportion of observations falling into each node, i.e.,  $\omega_m = \frac{n_m}{n}$, $\omega_l = \frac{n_l}{n}$ and $\omega_r = \frac{n_r}{n}$. Then, to get the importance for $j^{th}$ feature in a single tree, we add up all $\Delta(\theta^*_m)$ where the split is at the $j^{th}$ variable:
\begin{equation} \label{t}
\text{VI}^{\text{T}}_j = \sum_{m, j \in \theta^*_m}\Delta(\theta^*_m).
\end{equation}
Here the sum is taken over all non-terminal nodes of the tree, and we use the notation $j \in \theta^*_m$ to denote that the split is based on the $j^{th}$ feature. 

The notion of split-improvement for decision trees can be easily extended to Random Forests by taking the average across all trees. Suppose there are $B$ base learners in the forest, we could naturally define
\begin{equation} \label{si}
\text{VI}^{\text{RF}}_j = \frac{1}{B}\sum_{b = 1}^B\text{VI}^{\text{T(b)}}_j = \frac{1}{B}\sum_{b = 1}^B\sum_{m, j \in \theta^*_m}\Delta_b(\theta^*_m).
\end{equation}

\subsection{Bias in Split-Improvement}

\citet{strobl2007bias} pointed out that the split-improvement measure defined above is biased towards increasing the importance of continuous features or categorical features with many categories. This is because of the increased flexibility afforded by a larger number of potential split points. We conducted a similar simulation to further demonstrate this phenomenon. All our experiments are based on Random Forests which gives more stable results than a single tree.

We generate a simulated dataset so that $X_1 \sim N(0, 1)$ is continuous, and $X_2, X_3, X_4, X_5$ are categorically distributed with $2, 4, 10, 20$ categories respectively. The probabilities are equal across categories within each feature. In particular, $X_2$ is Bernoulli distribution with $p=0.5$. In classification setting, the response $y$ is also generated as a Bernoulli distribution with $p=0.5$, but independent of all the $X$'s. For regression, $y$ is independently generated as $N(0, 1)$. We repeat the simulation 100 times, each time generating $n=1000$ data points and fitting a Random Forest model\footnote{Our experiments are implemented using \textbf{scikit-learn}. Unless otherwise noted, default parameters are used.} using the data set. Here categorical features are encoded into dummy variables, and we sum up importance scores for corresponding dummy variables as final measurement for a specific categorical feature. In Appendix \ref{sim}, we also provide simulation results when treating those categorical features as (ordered) discrete variables. 

\begin{figure}
    \centering
    \begin{subfigure}[b]{0.45\textwidth}
        \centering
        \includegraphics[width=\textwidth]{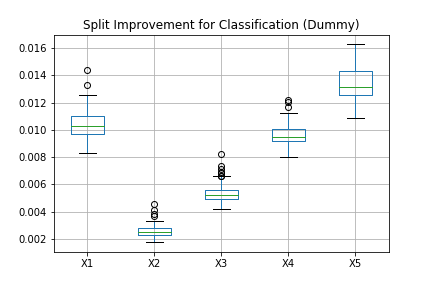}
        \caption{Classification}
        \label{fig:si_cls}
    \end{subfigure}
    \hfill
    \begin{subfigure}[b]{0.45\textwidth}
        \centering
        \includegraphics[width=\textwidth]{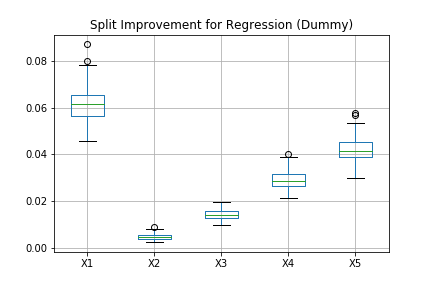}
        \caption{Regression}
        \label{fig:si_regr}
    \end{subfigure}
    \hfill
    
    \caption{Split-improvement measures on five predictors. Box plot is based on 100 repetitions. 100 trees are built in the forest and maximum depth of each tree is set to 5.}
    \label{fig:si}
    
\end{figure}

Box plots are shown in Figure \ref{fig:si_cls} and \ref{fig:si_regr} for classification and regression respectively. The continuous feature $X_1$ is frequently given the largest importance score in regression setting, and among the four categorical features, those with more categories receive larger importance scores. Similar phenomenon is observed in classification as well, while $X_5$ appears to be artificially more important than $X_1$. Also note that all five features get positive importance scores, though we know that they have no predictive power for the target value $y$.

We now explore how strong a signal is needed in order for the split-improvement measures to discover important predictors. We generate $X_1, X_2, \ldots, X_5$ as before, but in regression settings set $y = \rho X_2 + \epsilon $ where $\epsilon \sim N(0, 1)$. We choose $\rho$ to range from 0 to 1 at step size 0.1 to encode different levels of signal. For classification experiments, we first make $y = X_2$ and then flip each element of $y$ according to $P(U > \frac{1 + \rho}{2})$ where $U$ is Uniform$[0,1]$. This way, the correlation between $X_2$ and $y$ will be approximately $\rho$. We report the average ranking of all five variables across 100 repetitions for each $\rho$. The results are shown in Figure \ref{fig:rank}.

We see that $\rho$ needs to be larger than 0.2 to actually find $X_2$ is the most important predictor in our classification setting, while in regression this value increases to 0.6. And we also observe that a clear order exists for the remaining (all unimportant) four features. 

\begin{figure}
    \centering
    \begin{subfigure}[b]{0.45\textwidth}
        \centering
        \includegraphics[width=\textwidth]{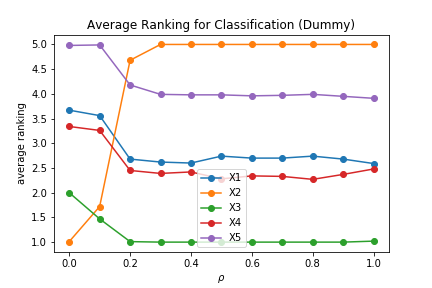}
        \caption{Classification}
        \label{fig:c_rank}
    \end{subfigure}
    \hfill
    \begin{subfigure}[b]{0.45\textwidth}
        \centering
        \includegraphics[width=\textwidth]{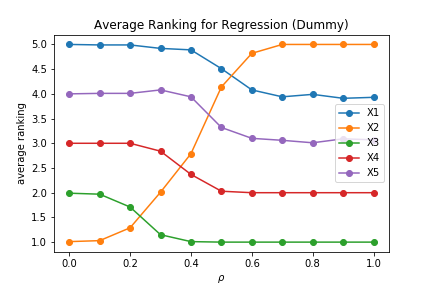}
        \caption{Regression}
        \label{fig:r_rank}
    \end{subfigure}
    \hfill
    
    \caption{Average feature importance ranking across different signal strengths over 100 repetitions. 100 trees are built in the forest and maximum depth of each tree is set to 5.}
    \label{fig:rank}
    
\end{figure}

This bias phenomenon could make many statistical analyses based on split-improvement invalid. For example, gender is a very common and powerful binary predictor in many applications, but feature screening based on split-improvement might think it is not important compared to age. In the next section, we explain intuitively why this bias is observed, and provide a simple but effective adjustment. 

\subsection{Related Work}

\zzz{Before presenting our algorithm, we review some related work aiming at correcting the bias in split-improvement. Most of the methods fall into two major categories: they either propose new tree building algorithms by redesigning split selection rules, or perform as a post hoc approach to debias importance measurement.}  

\zzz{There has been a line of work on designing trees which do not have such bias as observed in classical algorithms such as CART and C4.5. For example, Unbiased and Efficient Statistical Tree \citep[QUEST;][]{loh1997split} removed the bias by using F-tests on ordered variables and contingency table chi-squared tests on categorical variables. Based on QUEST, CRUISE \citep{kim2001classification} and GUIDE \citep{loh2009improving} were developed. We refer readers to \citet{loh2014fifty} for a detailed discussion in this aspect. In \citet{strobl2007bias}, the authors resorted to a different algorithm called cforest \citep{hothorn2010party}, which was based on a conditional inference framework \citep{hothorn2006unbiased}. They also implemented a stopping criteria based on multiple test procedures.} 

\zzz{\citet{sandri2008bias} expressed split-improvement as two components: a heterogeneity reduction and a positive bias. Then the original dataset $( \mathbf{X}, Y)$ is augmented with pseudo data $\mathbf{Z}$ which is uninformative but shares the structure of $\mathbf{X}$ \citep[this idea of generating pseudo data is later formulated in a general framework termed  ``knockoffs'';][]{barber2015controlling}. The positive bias term is estimated by utilizing the pseudo variables $\mathbf{Z}$ and subtracted to get a debiased estimate. \citet{nembrini2018revival} later modified this approach to shorten computation time and provided empirical importance testing procedures. Most recently, \citet{li2019debiased} derived a tight non-asymptotic bound on the expected bias of noisy features and provided a new debiased importance measure. However, this approach only alleviates the issue and still yields biased results.}

\zzz{Our approach works as a post hoc analysis, where the importance scores are calculated after a model is built. Compared to previous methods, it enjoys several advantages: 
\begin{itemize}
    \item It can be easily incorporated into any existing framework for tree-based methods, such as Python or R.
    \item It does not require generating additional pseudo data or computational repetitions as in \citet{sandri2008bias, nembrini2018revival}.
    \item Compared to \citet{li2019debiased} which does not have a theoretical guarantee, our method is proved to be unbiased for noisy features.
\end{itemize}}

\section{Unbiased Split-Improvement}\label{USI}

When it comes to evaluating the performance of machine learning models, we generally use a separate test set to calculate generalization accuracy. The training error is usually smaller than the test error as the algorithm is likely to "overfit" on the training data. This is exactly why we observe the bias with split-improvement. Each split will favor continuous features or those features with more categories, as they will have more flexibility to fit the training data. The vanilla version of split-improvement is just like using train error for evaluating model performance. 

Below we propose methods to remedy this bias phenomenon by utilizing a separate test set, and prove that for features with no predictive power, we're able to get an importance score of 0 in expectation for both classification and regressions settings. Our method is entirely based on the original framework of RF, requires barely no additional computational efforts, and can be easily integrated into any existing software libraries.

The main ingredient of the proposed method is to calculate the impurity function $H$ using additional information provided from test data. In the context of RF, we can simply take \textit{out-of-bag} samples for each individual tree. \new{Our experiments below are based on this strategy. In the context of the honest trees proposed in \citet{wager2018estimation} that divide samples into a partition used to determine tree structures and a partition used to obtain leaf values, the latter could be used as our test data below.  In boosting, it is common not to sample, but to keep a test set separate to determine a stopping time.} Since the choice of impurity function $H$ is different for classification and regression, in what follows we will treat them separately. 

Figure \ref{fig:test} and \ref{fig:urank} shows the results on previous classification and regression tasks when our unbiased method is applied\footnote{Relevant codes can be found at \href{https://github.com/ZhengzeZhou/unbiased-feature-importance}{https://github.com/ZhengzeZhou/unbiased-feature-importance}.}. Feature scores for all variables are spread around $0$, though continuous features and categorical features with more categories tend to exhibit more variability. In the case where there is correlation between $X_2$ and $y$, even for the smallest $\rho = 0.1$, we can still find the most informative predictor, whereas there are no clear order for the remaining noise features.

\begin{figure}
    \centering
    \begin{subfigure}[b]{0.45\textwidth}
        \centering
        \includegraphics[width=\textwidth]{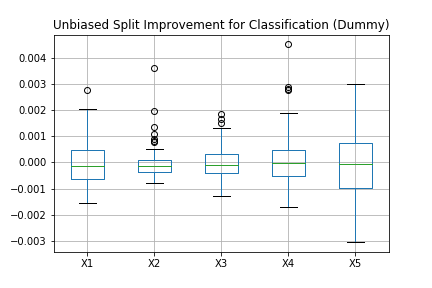}
        \caption{Classification}
        \label{fig:test_cls}
    \end{subfigure}
    \hfill
    \begin{subfigure}[b]{0.45\textwidth}
        \centering
        \includegraphics[width=\textwidth]{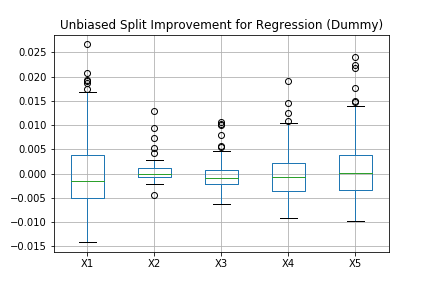}
        \caption{Regression}
        \label{fig:test_regr}
    \end{subfigure}
    \hfill
    
    \caption{Unbiased split-improvement. Box plot is based on 100 repetitions. 100 trees are built in the forest and maximum depth of each tree is set to 5. Each tree is trained using bootstrap samples and \textit{out-of-bag} samples are used as test set.}
    \label{fig:test}
    
\end{figure}

\begin{figure}
    \centering
    \begin{subfigure}[b]{0.45\textwidth}
        \centering
        \includegraphics[width=\textwidth]{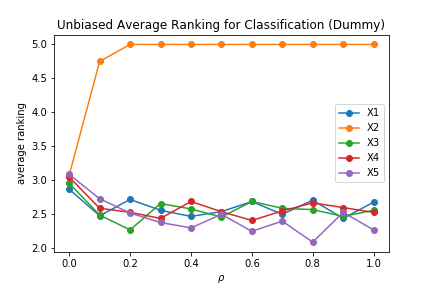}
        \caption{Classification}
        \label{fig:uc_rank}
    \end{subfigure}
    \hfill
    \begin{subfigure}[b]{0.45\textwidth}
        \centering
        \includegraphics[width=\textwidth]{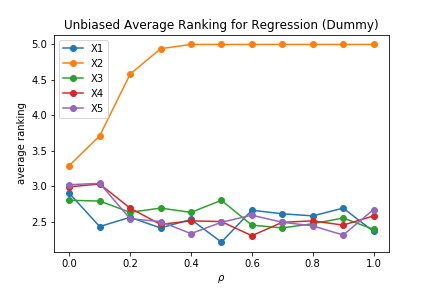}
        \caption{Regression}
        \label{fig:ur_rank}
    \end{subfigure}
    \hfill
    
    \caption{Unbiased feature importance ranking across different signal strengths averaged over 100 repetitions. 100 trees are built in the forest and maximum depth of each tree is set to 5. Each tree is trained using bootstrap samples and \textit{out-of-bag} samples are used as test set.}
    \label{fig:urank}
    
\end{figure}

\subsection{Classification}

Consider a root node $m$ and two child nodes, denoted by $l$ and $r$ respectively. The best split $\theta^*_m = (j, s)$ was chosen by Formula (\ref{best}) and Gini index is used as impurity function $H$. 

For simplicity, we focus on binary classification. Let $p$ denote class proportion within each node. For example, $p_{r, 2}$ denotes the proportion of class 2 in the right child node. Hence the Gini index for each node can be written as:
$$
H(m) = 1 - p_{m,1}^2 - p_{m,2}^2,
$$
$$
H(l) = 1 - p_{l,1}^2 - p_{l,2}^2,
$$
$$
H(r) = 1 - p_{r,1}^2 - p_{r,2}^2.
$$
The split-improvement for a split at $j^{th}$ feature when evaluated using only the training data is written as in Equation (\ref{di}). This value is always positive no matter which feature is chosen and where the split is, which is exactly why a selection bias will lead to overestimate of feature importance.

If instead, we have a separate test set available, the predictive impurity function for each node is modified to be:
\begin{equation}
  \label{pi}
  \begin{aligned}
    H'(m) &= 1 - p_{m,1}p'_{m,1} - p_{m,2}p'_{m,2}, \\        
    H'(l) &= 1 - p_{l,1}p'_{l,1} - p_{l,2}p'_{l,2}, \\
    H'(r) &= 1 - p_{r,1}p'_{r,1} - p_{r,2}p'_{r,2},
  \end{aligned}
\end{equation}
where $p'$ is class proportion evaluating on the test data.
And similarly, 
\begin{equation}
  \label{delta_c}
  \begin{aligned}
\Delta'(\theta^*_m) & = \omega_mH'(m) - (\omega_lH'(l) + \omega_rH'(r)) \\
& = \omega_l(H'(m) - H'(l)) + \omega_r(H'(m) - H'(r)).
  \end{aligned}
\end{equation}
\new{Using these definitions, we first demonstrate that an individual split is unbiassed in the sense that if $y$ has no bivariate relationship with $X_j$, $\Delta'(\theta^*_m)$ will have expectation 0.}
\begin{lemma}\label{l_c}
In classification settings, for a given feature $X_j$, if $y$ is marginally independent of $X_j$ within the region defined by node $m$, then 
$$
E\Delta'(\theta^*_m) = 0
$$
when splitting at the $j^{th}$ feature.
\end{lemma}

\begin{proof}
See Appendix \ref{proof}.
\end{proof}
Similar to Equation (\ref{t}), split-improvement of $x_j$ in a decision tree is defined as:
\begin{equation} \label{ut}
\text{VI}^{\text{T,C}}_j = \sum_{m, j \in \theta^*_m}\Delta'(\theta^*_m).
\end{equation}
\new{We can now apply Lemma \ref{l_c} to provide a global result so long as $X_j$ is always irrelevant to $y$.}
\begin{theorem}
In classification settings, for a given feature $X_j$, if $y$ is independent of $X_j$ in every hyper-rectangle subset of the feature space, then we always have
$$
E\text{VI}^{\text{T,C}}_j = 0.
$$
\end{theorem}

\begin{proof}
The result follows directly from Lemma \ref{l_c} and Equation (\ref{ut}). 
\end{proof}

This unbiasedness result can be easily extended to the case of RF by (\ref{si}), as it's an average across base learners. \new{We note here that our independence condition is designed to account for relationships that appear before accounting for splits on other variables, possibly due to relationships between $X_j$ and other features, and afterwards. It is trivially implied by the independence of $X_j$ with both $y$ and the other features. Our condition may also be stronger than necessary, depending on the tree-building process. We may be able to restrict the set of hyper-rectangles to be examined, but only by analyzing specific tree-building algorithms.}

\subsection{Regression}

In regression, we use mean squared error as the impurity function $H$:
$$
\bar{y}_m = \frac{1}{n_m}\sum_{x_i \in m}y_i,
$$
$$
H(m) = \frac{1}{n_m}\sum_{x_i \in m}(y_i - \bar{y}_m)^2.
$$
If instead the impurity function $H$ is evaluated on a separate test set, we define 
$$
H'(m) = \frac{1}{n'_m}\sum_{i=1}^{n'_m}(y'_{m,i} - \bar{y}_m)^2 
$$
and similarly
$$
\Delta'(\theta^*_m) =  \omega_mH'(m) - (\omega_lH'(l) + \omega_rH'(r)).
$$
Note that here $H'(m)$  measures mean squared error within node $m$ on test data with the fitted value $\bar{y}_m$ from training data. If we just sum up $\Delta'$ as feature importance, it will end up with negative values as $\bar{y}_m$ will overfit the training data and thus make mean squared error much larger deep in the tree. In other words, it \textit{over-corrects} the bias. For this reason, our unbiased split-improvement is defined slightly different from the classification case (\ref{ut}):
\begin{equation} \label{utr}
\text{VI}^{\text{T,R}}_j = \sum_{m, j \in \theta^*_m} (\Delta(\theta^*_m)  + \Delta'(\theta^*_m)).
\end{equation}

Notice that although Equation (\ref{ut}) and (\ref{utr}) are different, they originates from the same idea by correcting bias using test data. Unlike Formula (\ref{pi}) for Gini index, where we could design a predictive impurity function by combining train and test data together, it's hard to come up with a counterpart in regression setting. 

Just as in the classification case, we could show the following unbiasedness results:

\begin{lemma}\label{l_r}
In regression settings, for a given feature $X_j$, if $y$ is marginally independent of $X_j$ within the region defined by node $m$, then 
$$
E(\Delta(\theta^*_m) + \Delta'(\theta^*_m)) = 0
$$
when splitting at the $j^{th}$ feature.
\end{lemma}

\begin{proof}
See Appendix \ref{proof}.
\end{proof}
\begin{theorem}
In regression settings, for a given feature $X_j$, if $y$ is independent of $X_j$ in every hyper-rectangle subset of the feature space, then we always have
$$
E\text{VI}^{\text{T,R}}_j = 0.
$$
\end{theorem}

\section{Empirical Studies}\label{real}

\zzz{In this section, we apply our method to one simulated example and three real data sets. We compare our results to three other algorithms: the default split-improvement in \textbf{scikit-learn}, cforest \citep{hothorn2006unbiased} in R package \textbf{party} and bias-corrected impurity \citep{nembrini2018revival} in R package \textbf{ranger}. We did not include comparison with \citet{li2019debiased} since their method does not enjoy the unbiased property. In what follows, we use shorthand SI for the default split-improvement, UFI for our method (unbiased feature importance).}

\subsection{Simulated Data}

\zzz{The data has 1000 samples and 10 features, where $X_i$ takes values in $0, 1, 2, \ldots, i - 1$ with uniform probability for $1 \leq i \leq 10$. Here, we assume only $X_1$ contains true signal and all remaining nine features are noisy features. The target value $y$ is generated as follows:
\begin{itemize}
    \item Regression: $y = X_1 + 5 \epsilon$, where $\epsilon \sim \mathcal{N}(0, 1)$.
    \item Classification: $P(y = 1 | X) = 0.55$ if $X_1 = 1$, and $P(y = 1 | X) = 0.45$ if $X_1 = 0$.
\end{itemize}
Note that this task is designed to be extremely hard by choosing the binary feature as informative, and adding large noise (regression) or setting the signal strength low (classification). To evaluate the results, we look at the ranking of all features based on importance scores. Ideally $X_1$ should be ranked $1^{st}$ as it is the only informative feature. Table \ref{table:sum} shows the average ranking of feature $X_1$ across 100 repetitions. The best result of each column is marked in \textbf{bold}. Here we also compare the effect of tree depth by constructing shallow trees (with tree depth 3) and deep trees (with tree depth 10). Since cforest does not provide a parameter for directly controlling tree depth, we change the values of mincriterion as an alternative. }

\begin{table}[h]
\centering
\begin{tabular}{|c|c|c|c|c|}
\hline
\multirow{2}{*}{} & \multicolumn{2}{c|}{Tree depth = 3} & \multicolumn{2}{c|}{Tree depth = 10} \\ \cline{2-5} 
                  & R                & C                & R                 & C                \\ \hline
SI                & 3.71             & 4.10             & 10.00             & 10.00            \\ \hline
UFI               & \textbf{1.47}    & 1.39             & \textbf{1.55}     & \textbf{1.69}    \\ \hline
cforest           & 1.57             & \textbf{1.32}    & 1.77              & 1.88             \\ \hline
ranger            & 1.54             & 1.64             & 2.46              & 1.93             \\ \hline
\end{tabular}
\caption{Average importance ranking of informative feature $X_1$. R stands for regression and C for classification. The result averages over 100 repetitions. Lower values indicate better abilities in identifying informative features. In cforest, we set mincriterion to be 2.33 (0.99 percentile of normal distribution) for shallow trees and 1.28 (0.9 percentile) for deep trees.}
\label{table:sum}
\end{table}

\zzz{We can see that our method UFI achieves the best results in three situations except the classification case for shallow trees, where it is only slightly worse than cforest. Another interesting observation is that deeper trees tend to make the task of identifying informative features harder when there are noisy ones, since it is more likely to split on noisy features for splits deep down in the tree. This effect is most obvious for the default split-improvement, where it performs the worst especially for deep trees: the informative feature $X_1$ is consistently ranked as the least important ($10^{th}$ place). UFI does not seem to be affected too much from tree depth.}

\subsection{RNA Sequence Data}
The first data set examined is the prediction of C-to-U edited sites in plant mitochondrial RNA. This task was studied statistically in \citet{cummings2004simple}, where the authors applied Random Forests and used the original split-improvement as feature importance. Later, \citet{strobl2007bias} demonstrated the performance of cforest on this data set. 

RNA editing is a molecular process whereby an RNA sequence is
modified from the sequence corresponding to the DNA
template.  In the mitochondria of land plants, some cytidines
are converted to uridines before translation \citep{cummings2004simple}. 

We use the \textit{Arabidopsis thaliana} data file\footnote{The data set can be downloaded from \href{https://bmcbioinformatics.biomedcentral.com/articles/10.1186/1471-2105-5-132}{https://bmcbioinformatics.biomedcentral.com/articles/10.1186/1471-2105-5-132}.} as in \citet{strobl2007bias}. The features are based on the nucleotides surrounding the edited/non-edited sites and on the estimated folding energies of those regions. After removing missing values and one column which will not be used, the data file consists of 876 rows and 45 columns:
\begin{itemize}
    \item the response (binary).
    \item 41 nucleotides at positions -20 to 20 relative to the
edited site (categorical, one of A, T, C or G).
    \item the codon position (also 4 categories).
    \item two continuous variables based on on the estimated folding energies.
\end{itemize} 

For implementation, we create dummy variables for all categorical features, and build forest using 100 base trees. The maximum tree depth for this data set is not restricted as the number of potential predictors is large. We take the sum of importance across all dummy variables corresponding to a specific feature for final importance scores. All default parameters are used unless otherwise specified.

The results are shown in Figure \ref{fig:ctou}. Red error bars depict one standard deviation when the experiments are repeated 100 times. From the default split-improvement (Figure \ref{fig:si_ctou}), we can see that except several apparently dominant predictors (nucleotides at position -1 and 1, and two continuous features \textit{fe} and \textit{dfe}), the importance for the remaining nearly 40 features are indistinguishable. The feature importance scores given by UFI (Figure \ref{fig:ufi_ctou}) and cforest (Figure \ref{fig:cf_ctou}) are very similar. Compared with SI, although all methods agree on top three features being the nucleotides at position -1 and 1, and the continuous one \textit{fe}, there are some noticeable differences. Another continuous feature \textit{dfe} is originally ranked at the fourth place in Figure \ref{fig:si_ctou}, but its importance scores are much lower by UFI and cforest. The result given by ranger (Figure \ref{fig:ranger_ctou}) is slightly different from UFI and cforest, where it seems to have more features with importance scores larger than 0. In general, we see a large portion of predictors with feature importance close to 0 for three improved methods, which makes subsequent tasks like feature screening easier.

\begin{figure}
    \centering
    \begin{subfigure}[b]{0.495\textwidth}
        \centering
        \includegraphics[width=\textwidth]{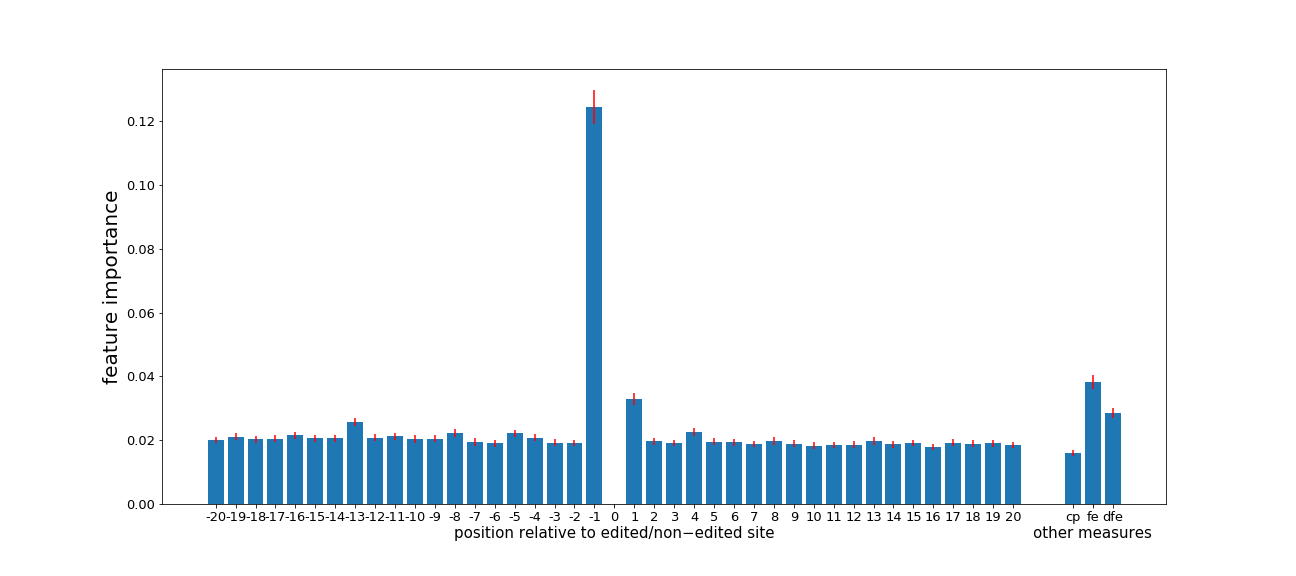}
        \caption{SI}
        \label{fig:si_ctou}
    \end{subfigure}
    \hfill
    \begin{subfigure}[b]{0.495\textwidth}
        \centering
        \includegraphics[width=\textwidth]{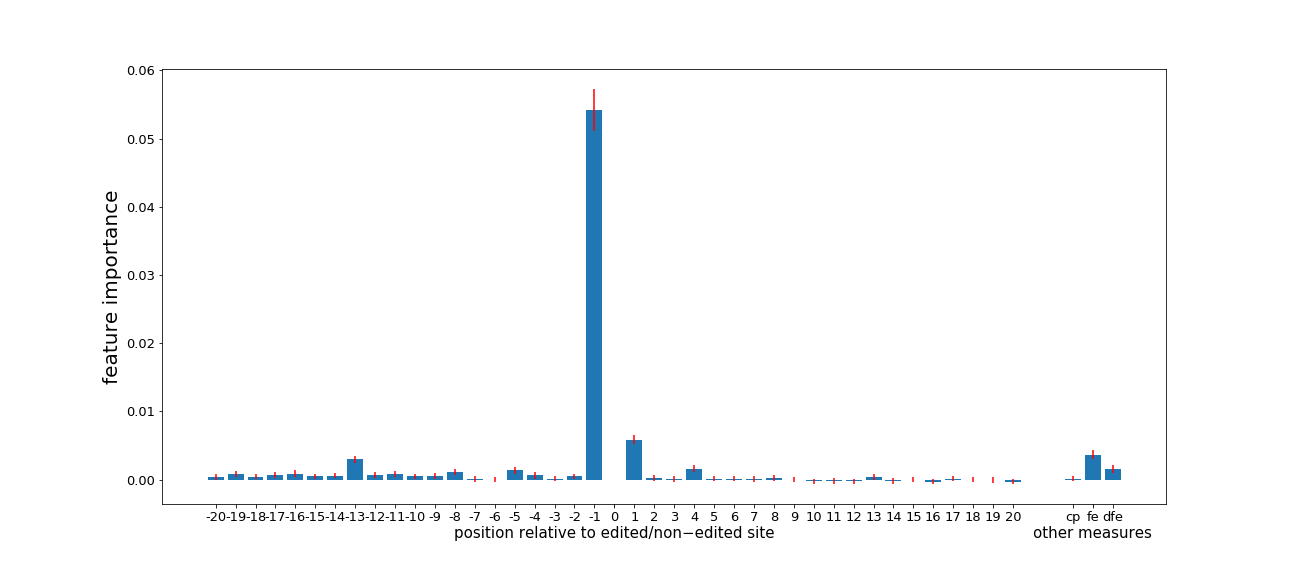}
        \caption{UFI}
        \label{fig:ufi_ctou}
    \end{subfigure}
    \hfill

    \begin{subfigure}[b]{0.495\textwidth}
        \centering
        \includegraphics[width=\textwidth]{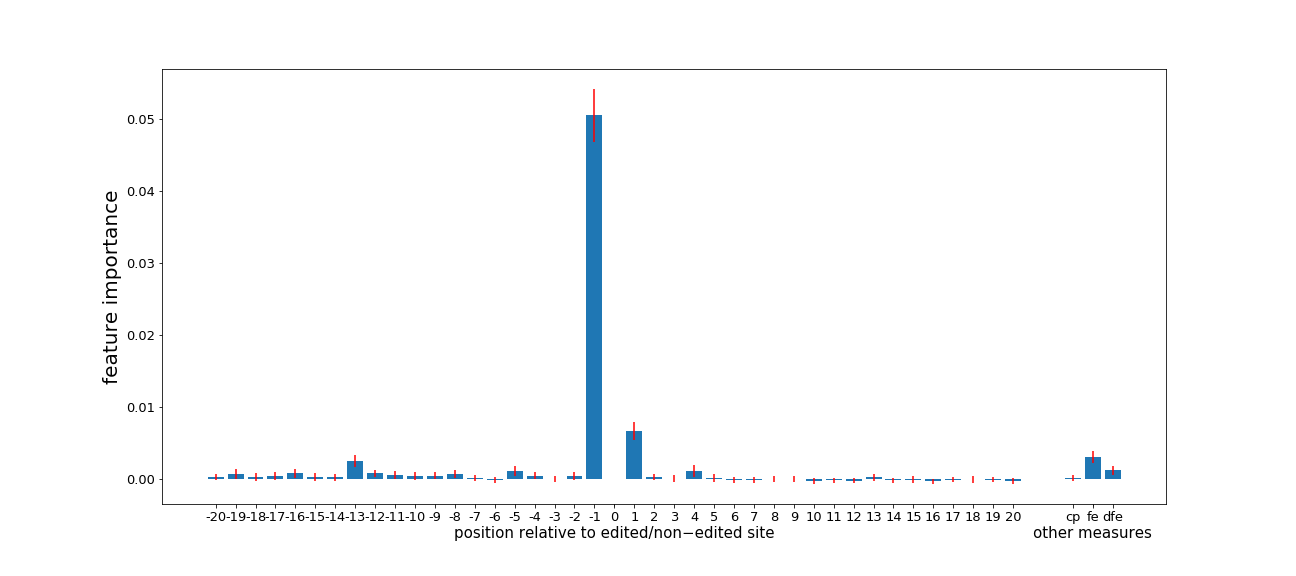}
        \caption{cforest}
        \label{fig:cf_ctou}
    \end{subfigure}
    \hfill
    \begin{subfigure}[b]{0.495\textwidth}
        \centering
        \includegraphics[width=\textwidth]{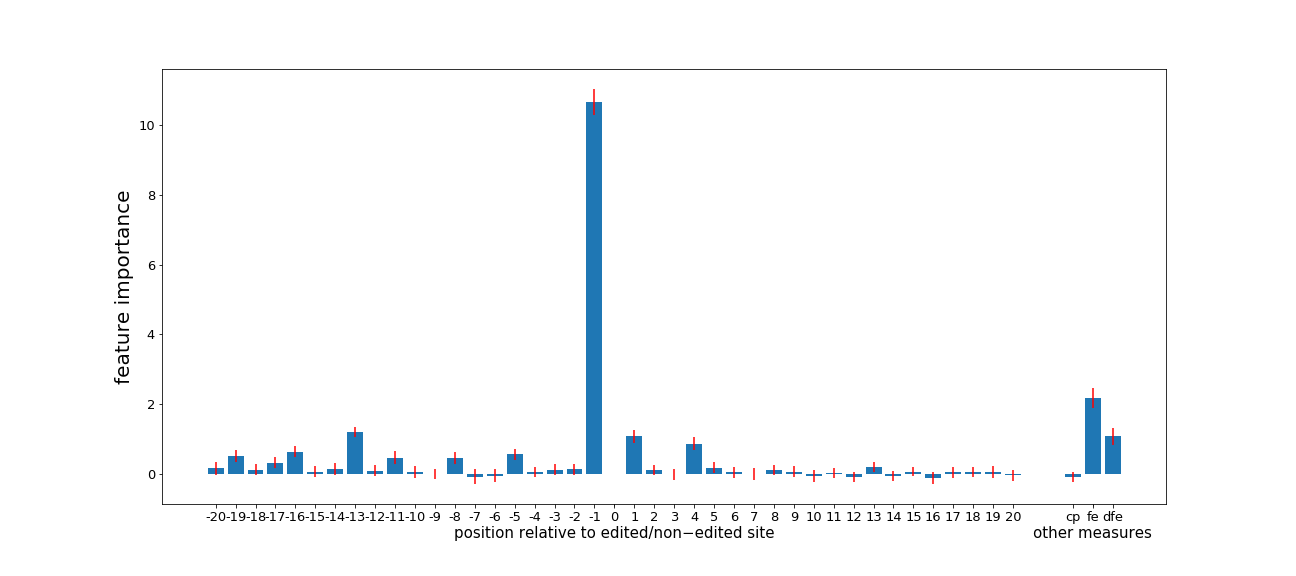}
        \caption{ranger}
        \label{fig:ranger_ctou}
    \end{subfigure}
    
    \caption{Feature importance for RNA sequence data. 100 trees are built in the forest. Red error bars depict one standard deviation when the experiments are repeated 100 times.}
    \label{fig:ctou}
    
\end{figure}

\subsection{Adult Data}

As a second example, we will use the Adult Data Set from UCI Machine Learning Repository\footnote{\href{https://archive.ics.uci.edu/ml/datasets/adult}{https://archive.ics.uci.edu/ml/datasets/adult}}. The task is to predict whether income exceeds \$50K/yr based on census data. We remove all entries including missing values, and only focus on people from Unites States. In total, there are 27504 training samples and Table \ref{table:adult} describes relevant feature information. Notice that we add a standard normal random variable, which is shown in the last row. We randomly sample 5000 entries for training. 

\begin{table}[]
\centering
\begin{tabular}{|c|c|}
\hline
Attribute      & Description      \\ \hline
age            & continuous       \\ \hline
workclass      & categorical (7)  \\ \hline
fnlwgt         & continuous       \\ \hline
education      & categorical (16) \\ \hline
education-num  & continuous       \\ \hline
marital-status & categorical (7)  \\ \hline
occupation     & categorical (14) \\ \hline
relationship   & categorical (6)  \\ \hline
race           & categorical (5)  \\ \hline
sex            & binary           \\ \hline
capital-gain   & continuous       \\ \hline
capital-loss   & continuous       \\ \hline
hours-per-week & continuous       \\ \hline
random         & continuous       \\ \hline
\end{tabular}
\caption{Attribute description for adult data set.}
\label{table:adult}
\end{table}

The results are shown in Figure \ref{fig:adult}. UFI (\ref{fig:ufi_adult}), cforest(\ref{fig:cf_adult}) and ranger (\ref{fig:ranger_adult}) display similar feature rankings which are quite different from the original split-improvement (\ref{fig:si_adult}). Notice the random normal feature we added (marked in black) is actually ranked the third most important in \ref{fig:si_adult}. This is not surprising as most of the features are categorical, and even for some continuous features, a large portion of the values are actually 0 (such as \textit{capital-gain} and \textit{capital-loss}). For UFI, cforest and ranger, the random feature is assigned an importance score close to 0. Another feature with big discrepancy is \textit{fnlwgt}, which is ranked among top three originally but is the least important for other methods. \textit{fnlwgt} represents final weight, the number of units in the target population that the responding unit represents. Thus it is unlikely to have strong predictive power for the response. For this reason, some analyses deleted this predictor before fitting models\footnote{\href{http://scg.sdsu.edu/dataset-adult\_r/}{http://scg.sdsu.edu/dataset-adult\_r/}}.

\begin{figure}
    \centering
    \begin{subfigure}[b]{0.49\textwidth}
        \centering
        \includegraphics[width=\textwidth]{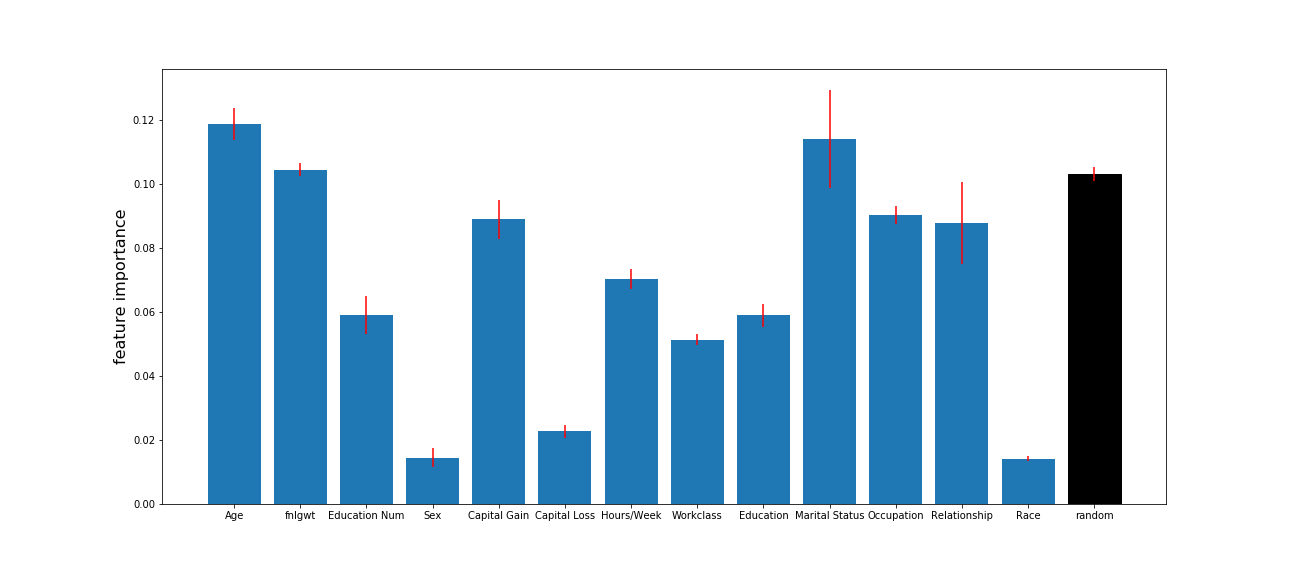}
        \caption{SI}
        \label{fig:si_adult}
    \end{subfigure}
    \hfill
    \begin{subfigure}[b]{0.49\textwidth}
        \centering
        \includegraphics[width=\textwidth]{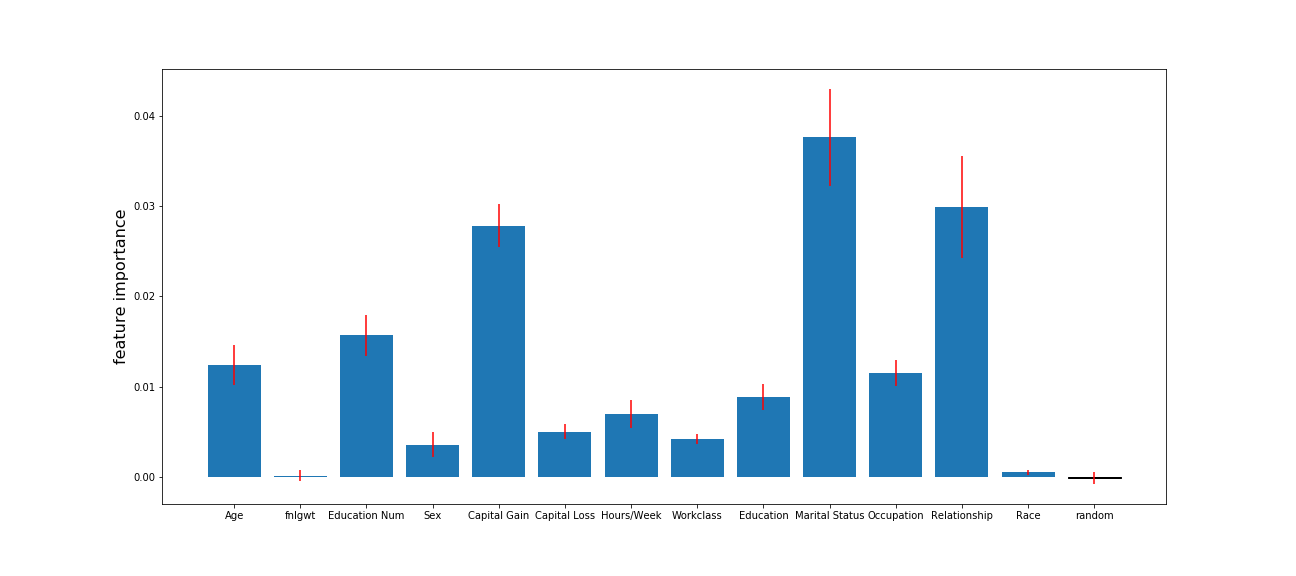}
        \caption{UFI}
        \label{fig:ufi_adult}
    \end{subfigure}
    \hfill
    
    \begin{subfigure}[b]{0.49\textwidth}
    \centering
    \includegraphics[width=\textwidth]{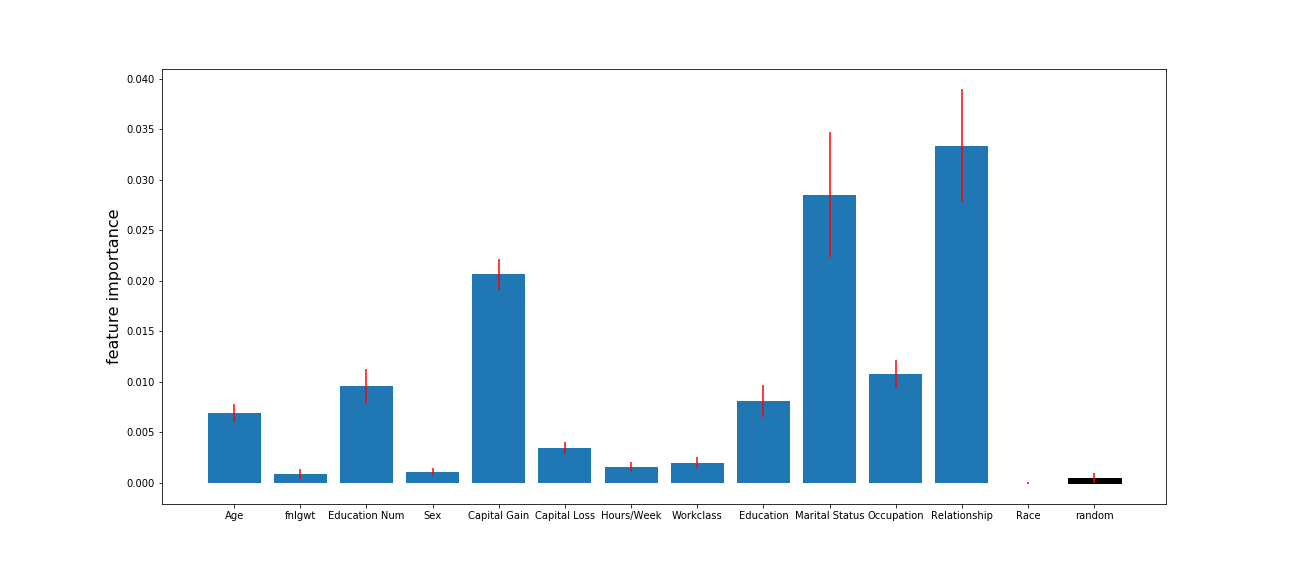}
    \caption{cforest}
    \label{fig:cf_adult}
    \end{subfigure}
    \hfill
    \begin{subfigure}[b]{0.49\textwidth}
        \centering
        \includegraphics[width=\textwidth]{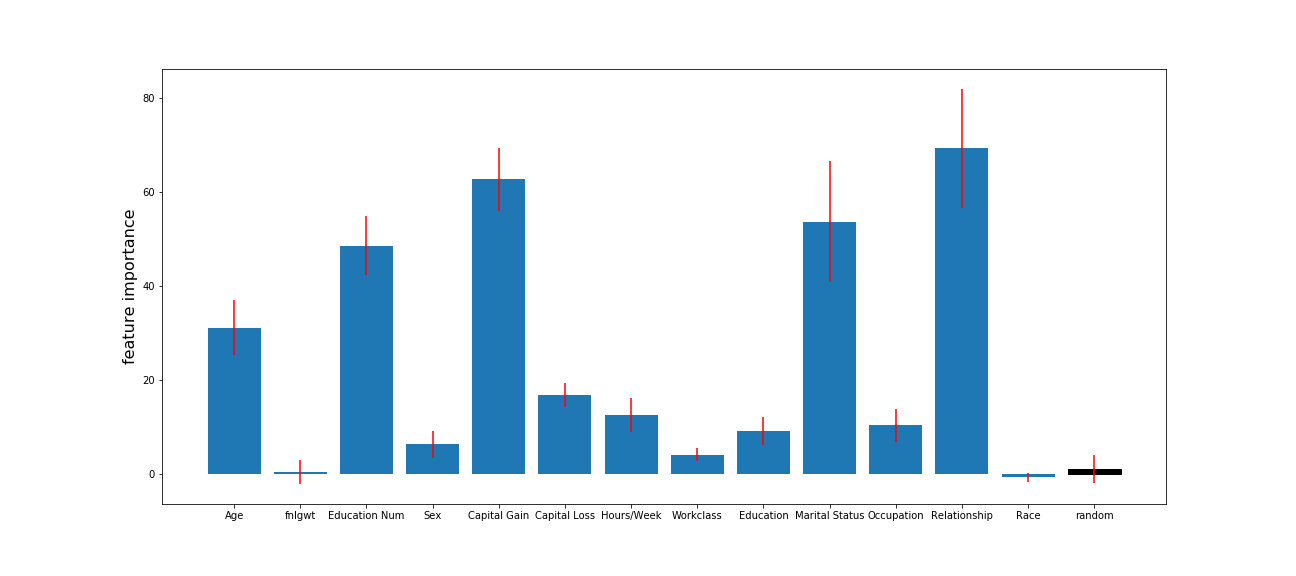}
        \caption{ranger}
        \label{fig:ranger_adult}
    \end{subfigure}
    
    \caption{Feature importance for adult data. 20 trees are built in the forest. Red error bars depict one standard deviation when the experiments are repeated 100 times.}
    \label{fig:adult}
\end{figure}

\subsection{Boston Housing Data}

\zzz{We also conduct analyses on a regression example using the Boston Housing Data\footnote{\href{https://archive.ics.uci.edu/ml/machine-learning-databases/housing/}{https://archive.ics.uci.edu/ml/machine-learning-databases/housing/}}, which has been widely studied in previous literature \citep{bollinger1981book, quinlan1993combining}. The data set contains 12 continuous, one ordinal and one binary features and the target is median value of owner-occupied homes in \$1000's. We add a random feature distributed as $\mathcal{N}(0, 1)$ as well.}

\begin{figure}[h]
    \centering
    \begin{subfigure}[b]{0.49\textwidth}
        \centering
        \includegraphics[width=\textwidth]{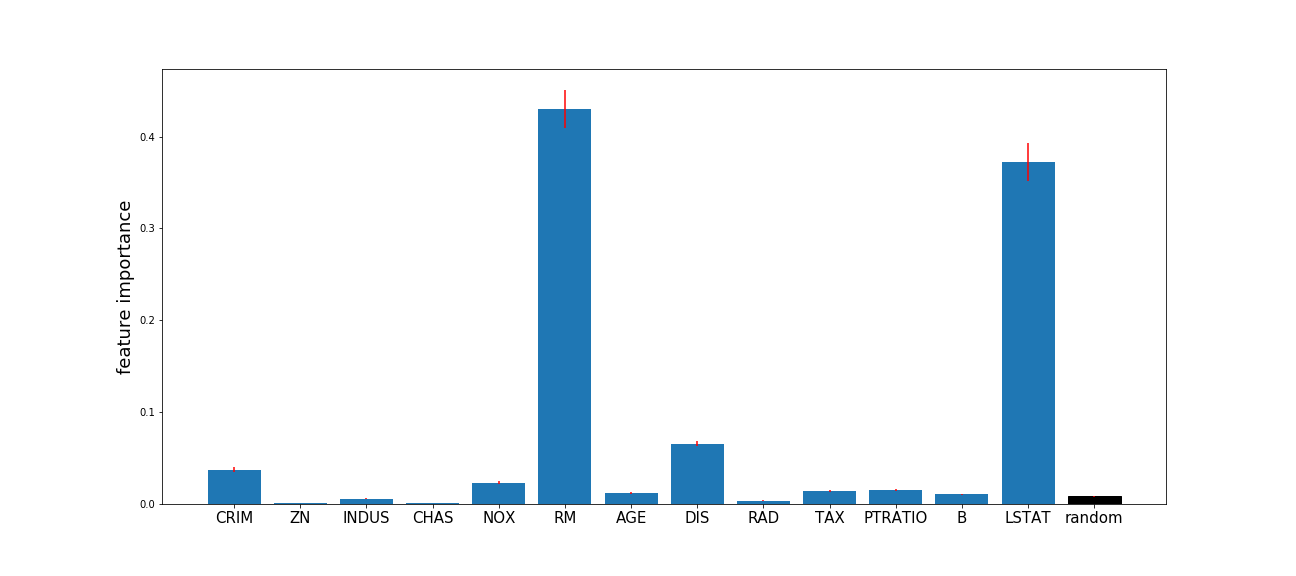}
        \caption{SI}
        \label{fig:si_boston}
    \end{subfigure}
    \hfill
    \begin{subfigure}[b]{0.49\textwidth}
        \centering
        \includegraphics[width=\textwidth]{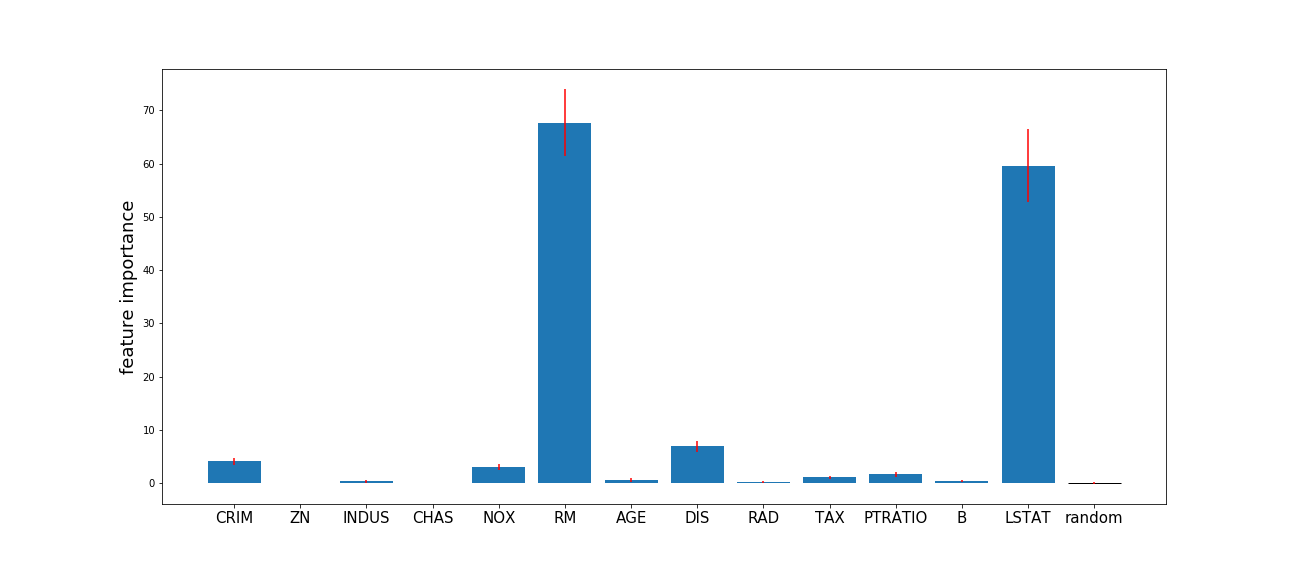}
        \caption{UFI}
        \label{fig:ufi_boston}
    \end{subfigure}
    \hfill
    
    \begin{subfigure}[b]{0.49\textwidth}
    \centering
    \includegraphics[width=\textwidth]{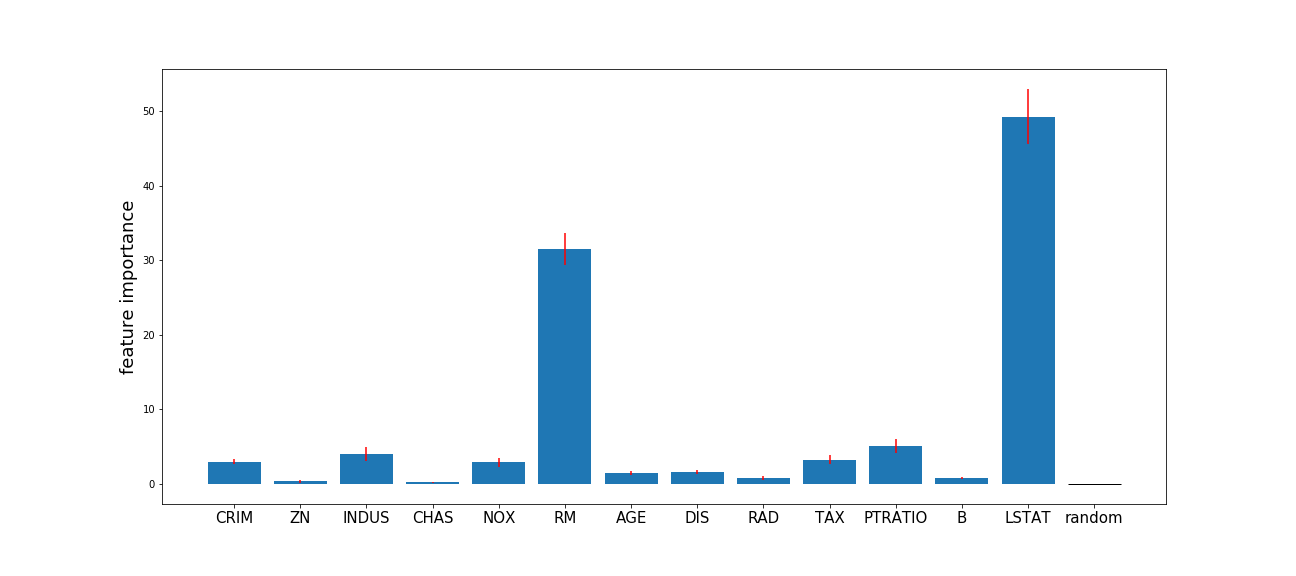}
    \caption{cforest}
    \label{fig:cf_boston}
    \end{subfigure}
    \hfill
    \begin{subfigure}[b]{0.49\textwidth}
        \centering
        \includegraphics[width=\textwidth]{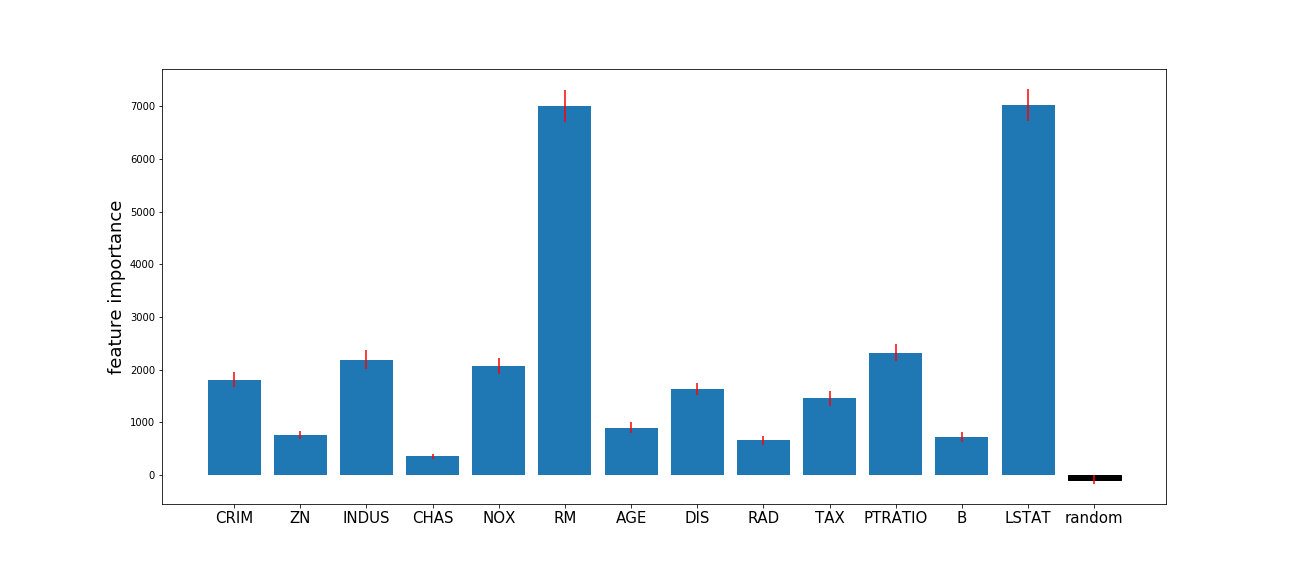}
        \caption{ranger}
        \label{fig:ranger_boston}
    \end{subfigure}
    
    \caption{Feature importance for Boston housing data. 100 trees are built in the forest. Red error bars depict one standard deviation when the experiments are repeated 100 times.}
    \label{fig:boston}
\end{figure}

\zzz{All four methods agree on two most important features: RM (average number of rooms per dwelling) and LSTAT (\% lower status of the population). In SI, the random feature still appears to be more important than several other features such as INDUS (proportion of non-retail business acres per town) and RAD (index of accessibility to radial highways), though the spurious effect is much less compared to Figure \ref{fig:si_adult}. As expected, the importance of random feature is close to zero in UFI. In this example, the SI did not seem to provide misleading result as most of the features are continuous, and the only binary feature CHAS (Charles River dummy variable) turns out to be not important. }

\subsection{Summary}

\zzz{Our empirical studies confirm that the default split-improvement method is biased towards increasing the importance of features with more potential splits. The bias is more severe in deeper trees. Compared to three other approaches, our proposed method performs the best in a difficult task to identify the only important feature from 10 noisy features. For real world data sets, though we do not have a ground truth for feature importance scores, our method gives similar and meaningful outputs as two state-of-the-art methods cforest and ranger.}

\section{Discussions}\label{dis}

Tree-based methods are widely employed in many applications. One of the many advantages is that these models come naturally with feature importance measures, which practitioners rely on heavily for subsequent analysis such as feature ranking or screening. It is important that these measurements are trustworthy.

We show empirically that split-improvement, as a popular measurement of feature importance in tree-based models, is biased towards continuous features, or categorical features with more categories. This phenomenon is akin to overfitting in training any machine learning model. We propose a simple fix to this problem and demonstrate its effectiveness both theoretically and empirically. Though our examples are based on Random Forests, the adjustment can be easily extended to any other tree-based model. 

The original version of split-improvement is the default and only feature importance measure for Random Forests in \textbf{scikit-learn}, and is also returned as one of the measurements for \textbf{randomForest} library in R. Statistical analyses utilizing these packages will suffer from the bias discussed in this paper. Our method can be easily integrated into existing libraries, and require almost no additional computational burden.  \new{As already observed, while we have used \textit{out-of-bag} samples as a natural source of test data,  alternatives such as sample partitions -- thought of as a subsample of \textit{out-of-bag} data for our purposes -- can be used in the context of honest trees, or a held-out test set will also suffice.   The use of subsamples fits within the methods used to demonstrate the asymptotic normality of Random Forests developed in \citet{mentch2016quantifying}. This potentially allows for formal statistical tests to be developed based on the unbiased split-improvement measures proposed here.}  Similar approaches have been taken in \citet{zhou2018approximation} for designing stopping rules in approximation trees. 



However, feature importance itself is very difficult to define exactly, with the possible exception of linear models, where the magnitude of coefficients serves as a simple measure of importance. There are also considerable discussion on the subtly introduced when correlated predictors exist, see for example \cite{strobl2008conditional, gregorutti2017correlation}. We think \new{that clarifying the relationship between split-improvement and the topology of the resulting function represents an important} future research direction. 

\subsection*{Acknowledgements}
This work was supported in part by NSF grants DMS-1712554, DEB-1353039 and TRIPODS 1740882. 

\bibliographystyle{ACM-Reference-Format}
\bibliography{sample}

%
\appendix

\section{Proofs of Lemma \ref{l_c} and \ref{l_r}}\label{proof}

\begin{proof}[Proof of Lemma \ref{l_c}]
We want to show that for independent $X_{j}$ and $y$ within node $m$, $\Delta'(\theta^*_m)$ should ideally be zero when splitting on the $j^{th}$ variable. Rewriting $H'(m)$ defined in Equation (\ref{pi}) and we get:
\begin{align*}
H'(m) &= 1 - p_{m,1}p'_{m,1} - p_{m,2}p'_{m,2} \\
&= 1 - p_{m,1}p'_{m,1} - (1 - p_{m,1})(1 - p'_{m,1}) \\
&= p_{m,1} + p'_{m,1} - 2p_{m,1}p'_{m,1}.
\end{align*}
Using similar expressions for $H'(l)$, we have:
$$
H'(m) - H'(l) = (p_{m,1} + p'_{m,1} - 2p_{m,1}p'_{m,1}) - (p_{l,1} + p'_{l,1} - 2p_{l,1}p'_{l,1}).
$$
Given that the test data is independent of the training data and the independence between $X_{j}$ and $y$, then in expectation, we should have $E(p'_{m,1}) = E(p'_{l,1}) = p'_1$. Thus,
\begin{align*}
E(H'(m) - H'(l)) &= (E(p_{m,1}) + E(p'_{m,1}) - 2E(p_{m,1}p'_{m,1})) - (E(p_{l,1}) + E(p'_{l,1}) - 2E(p_{l,1}p'_{l,1})) \\
&= (E(p_{m,1}) + E(p'_{m,1}) - 2E(p_{m,1})E(p'_{m,1})) - (E(p_{l,1}) + E(p'_{l,1}) - 2E(p_{l,1})E(p'_{l,1})) \\
&= (E(p_{m,1}) + p'_1 - 2E(p_{m,1})p'_1) - (E(p_{l,1}) + p'_1 - 2E(p_{l,1})p'_1) \\
&= (E(p_{m,1}) - E(p_{l,1}))(1 - 2p'_1).
\end{align*}
Similarly,
$$
E(H'(m) - H'(r)) = (E(p_{m,1}) - E(p_{r,1}))(1 - 2p'_1).
$$
Combined together into Equation (\ref{delta_c}), 
\begin{align*}
E(\Delta'(\theta^*_m)) &=\omega_l(H'(m) - H'(l)) + \omega_r(H'(m) - H'(r)) \\
&= \omega_l(E(p_{m,1}) - E(p_{l,1}))(1 - 2p'_1) + \omega_r(E(p_{m,1}) - E(p_{r,1}))(1 - 2p'_1) \\
&= (1 - 2p'_1)(\omega_mE(p_{m,1}) - \omega_lE(p_{l,1}) - \omega_rE(p_{r,1})) \\
&= (1 - 2p'_1) \times 0\\
&= 0,
\end{align*}
since we always have
$$
\omega_m \times p_{m,1} = \omega_l \times p_{l,1} + \omega_r \times p_{r,1}.
$$
\end{proof}

\begin{proof}[Proof of Lemma \ref{l_r}]
Rewriting the expression of $H(m)$:
\begin{align*}
H(m) &= \frac{1}{n_m}\sum_{i=1}^{n_m}(y_{m,i} - \bar{y}_m)^2 \\
&= \frac{1}{n_m}(\sum_{i=1}^{n_m}y_{m,i}^2 - n_m\bar{y}_m^2).
\end{align*}
Thus, 
\begin{align*}
\Delta(\theta^*_m) &=  \omega_mH(m) - (\omega_lH(l) + \omega_rH(r)) \\
&= \omega_m\frac{1}{n_m}\sum_{i=1}^{n_m}(y_{m,i}^2 - n_m\bar{y}_m^2) - (\omega_l\frac{1}{n_l}\sum_{i=1}^{n_l}(y_{l,i}^2 - n_l\bar{y}_l^2) + \omega_r\frac{1}{n_r}\sum_{i=1}^{n_r}(y_{r,i}^2 - n_r\bar{y}_r^2)) \\
&= \frac{1}{n}\sum_{i=1}^{n_m}(y_{m,i}^2 - n_m\bar{y}_m^2) - (\frac{1}{n}\sum_{i=1}^{n_l}(y_{l,i}^2 - n_l\bar{y}_l^2) + \frac{1}{n}\sum_{i=1}^{n_r}(y_{r,i}^2 - n_r\bar{y}_r^2)) \\
&= \frac{1}{n}(\sum_{i=1}^{n_m}(y_{m,i}^2 - n_m\bar{y}_m^2) - \sum_{i=1}^{n_l}(y_{l,i}^2 - n_l\bar{y}_l^2) - \sum_{i=1}^{n_r}(y_{r,i}^2 - n_r\bar{y}_r^2)) \\
&= \frac{1}{n}(\sum_{i=1}^{n_m}y_{m,i}^2 - \sum_{i=1}^{n_l}y_{l,i}^2 - \sum_{i=1}^{n_r}y_{r,i}^2) - \frac{1}{n}(n_m\bar{y}_m^2 - n_l\bar{y}_l^2 - n_r\bar{y}_r^2) \\
&= \frac{1}{n}(n_l\bar{y}_l^2 + n_r\bar{y}_r^2 - n_m\bar{y}_m^2) \\
&= \omega_l\bar{y}_l^2 + \omega_r\bar{y}_r^2 - \omega_m\bar{y}_m^2.
\end{align*}
By Cauchy–Schwarz inequality, 
$$
(n_l\bar{y}_l^2 + n_r\bar{y}_r^2)(n_l + n_r) \geq (n_l\bar{y}_l + n_r\bar{y}_r)^2 = (n_m\bar{y}_m)^2,
$$
thus
$$
\Delta(\theta^*_m) = \frac{1}{n}(n_l\bar{y}_l^2 + n_r\bar{y}_r^2 - n_m\bar{y}_m^2) \geq 0
$$
unless $\bar{y}_l = \bar{y}_r = \bar{y}_m$.

Similarly for $H'(m)$:
\begin{align*}
H'(m) &= \frac{1}{n'_m}\sum_{i=1}^{n'_m}(y'_{m,i} - \bar{y}_m)^2 \\
&= \frac{1}{n'_m}\sum_{i=1}^{n'_m}y'^2_{m,i} - 2\frac{1}{n'_m}\sum_{i=1}^{n'_m}y'_{m,i}\bar{y}_m + \frac{1}{n'_m}\sum_{i=1}^{n'_m}\bar{y}_m^2 \\
&= \frac{1}{n'_m}\sum_{i=1}^{n'_m}y'^2_{m,i} - 2\bar{y}'_m\bar{y}_m + \bar{y}_m^2
\end{align*}
and thus
\begin{align*}
\Delta'(\theta^*_m) &=  \omega_mH'(m) - (\omega_lH'(l) + \omega_rH'(r)) \\
&= \omega_m(\frac{1}{n_m'}\sum_{i=1}^{n'_m}y'^2_{m,i} - 2\bar{y}'_m\bar{y}_m + \bar{y}_m^2) - \omega_l(\frac{1}{n'_l}\sum_{i=1}^{n'_l}y'^2_{l,i} - 2\bar{y}'_l\bar{y}_l + \bar{y}_l^2) - \omega_r(\frac{1}{n'_r}\sum_{i=1}^{n'_r}y'^2_{r,i} - 2\bar{y}'_r\bar{y}_r + \bar{y}_r^2) \\
&= (\omega_m\frac{1}{n_m'}\sum_{i=1}^{n'_m}y'^2_{m,i} - \omega_l\frac{1}{n_l'}\sum_{i=1}^{n'_l}y'^2_{l,i} - \omega_r\frac{1}{n_r'}\sum_{i=1}^{n'_r}y'^2_{r,i}) + (\omega_m\bar{y}_m^2 - \omega_l\bar{y}_l^2 - \omega_r\bar{y}_r^2) - 2(\omega_m\bar{y}'_m\bar{y}_m - \omega_l\bar{y}'_l\bar{y}_l - \omega_r\bar{y}'_r\bar{y}_r) \\
&= (\omega_m\frac{1}{n_m'}\sum_{i=1}^{n'_m}y'^2_{m,i} - \omega_l\frac{1}{n_l'}\sum_{i=1}^{n'_l}y'^2_{l,i} - \omega_r\frac{1}{n_r'}\sum_{i=1}^{n'_r}y'^2_{r,i}) - \Delta(\theta^*_m) - 2(\omega_m\bar{y}'_m\bar{y}_m - \omega_l\bar{y}'_l\bar{y}_l - \omega_r\bar{y}'_r\bar{y}_r).
\end{align*}

By the independence assumptions, we have 
$$
E\frac{1}{n_m'}\sum_{i=1}^{n'_m}y'^2_{m,i} = E \frac{1}{n_l'}\sum_{i=1}^{n'_l}y'^2_{l,i} = E \frac{1}{n_r'}\sum_{i=1}^{n'_r}y'^2_{r,i},
$$
and 
$$
E \bar{y}'_m = E \bar{y}'_l = E \bar{y}'_r.
$$

We can conclude that 
$$
E(\Delta(\theta^*_m) + \Delta'(\theta^*_m)) = 0.
$$

\end{proof}

\section{Additional Simulation Results}\label{sim}

Our simulation experiments in Section \ref{FI} and \ref{USI} operate by creating dummy variables for categorical features. It would be interesting to see the results if we instead treat those as ordered discrete values. 

Figure \ref{fig:si_dis} and \ref{fig:rank_dis} show the original version of split-improvement corresponding to Figure \ref{fig:si} and \ref{fig:rank}. Similar phenomenon is again observed: it over estimates importance of continuous features and categorical features with more categories. It is worth noticing that the discrepancy between continuous and categorical features is even larger in this case. Unlike in Figure \ref{fig:si_cls}, $X_1$ is always ranked the most important. This results from the fact that by treating categorical features as ordered discrete ones, it limit the number of potential splits compared to using dummy variables. 

\begin{figure}
    \centering
    \begin{subfigure}[b]{0.45\textwidth}
        \centering
        \includegraphics[width=\textwidth]{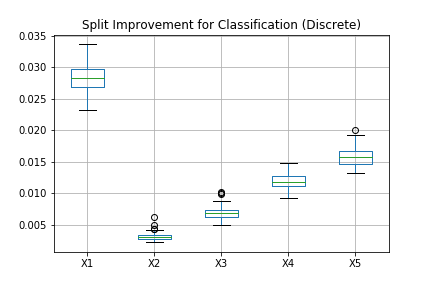}
        \caption{Classification}
        \label{fig:si_cls_dis}
    \end{subfigure}
    \hfill
    \begin{subfigure}[b]{0.45\textwidth}
        \centering
        \includegraphics[width=\textwidth]{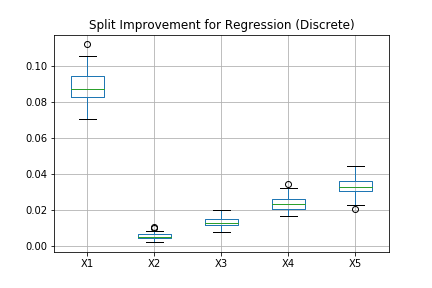}
        \caption{Regression}
        \label{fig:si_regr_dis}
    \end{subfigure}
    \hfill
    
    \caption{Split-improvement measures on five predictors, where we treat categorical features as ordered discrete values. Box plot is based on 100 repetitions. 100 trees are built in the forest and maximum depth of each tree is set to 5.}
    \label{fig:si_dis}
    
\end{figure}

\begin{figure}
    \centering
    \begin{subfigure}[b]{0.45\textwidth}
        \centering
        \includegraphics[width=\textwidth]{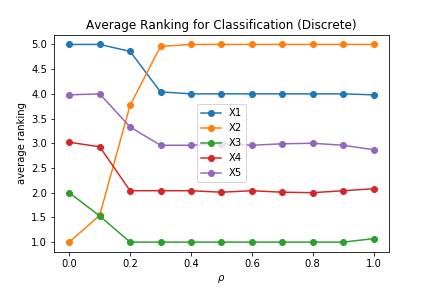}
        \caption{Classification}
        \label{fig:c_rank_dis}
    \end{subfigure}
    \hfill
    \begin{subfigure}[b]{0.45\textwidth}
        \centering
        \includegraphics[width=\textwidth]{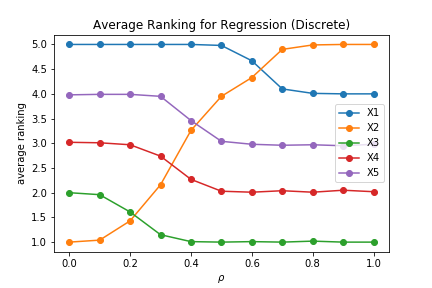}
        \caption{Regression}
        \label{fig:r_rank_dis}
    \end{subfigure}
    \hfill
    
    \caption{Average feature importance ranking across different signal strengths over 100 repetitions, where we treat categorical features as ordered discrete values. 100 trees are built in the forest and maximum depth of each tree is set to 5.}
    \label{fig:rank_dis}
    
\end{figure}

Not surprisingly, our proposed method work well in declaring all five features have no predictive power or finding the most informative one, as shown in Figure \ref{fig:test_dis} and \ref{fig:urank_dis}.

\begin{figure}
    \centering
    \begin{subfigure}[b]{0.45\textwidth}
        \centering
        \includegraphics[width=\textwidth]{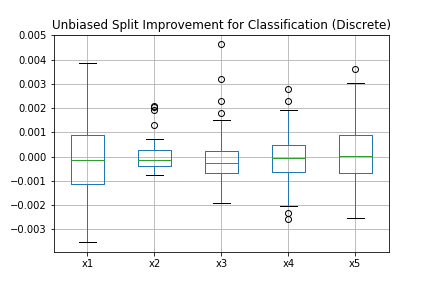}
        \caption{Classification}
        \label{fig:test_cls_dis}
    \end{subfigure}
    \hfill
    \begin{subfigure}[b]{0.45\textwidth}
        \centering
        \includegraphics[width=\textwidth]{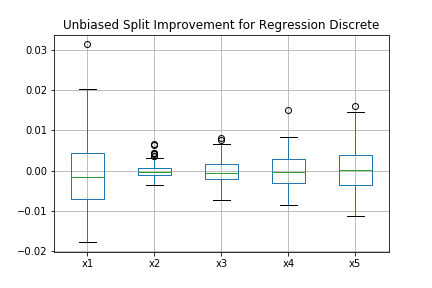}
        \caption{Regression}
        \label{fig:test_regr_dis}
    \end{subfigure}
    \hfill
    
    \caption{Unbiased split-improvement, where we treat categorical features as ordered discrete values. Box plot is based on 100 repetitions. 100 trees are built in the forest and maximum depth of each tree is set to 5. Each tree is trained using bootstrap samples and \textit{out-of-bag} samples are used as test set.}
    \label{fig:test_dis}
    
\end{figure}

\begin{figure}
    \centering
    \begin{subfigure}[b]{0.45\textwidth}
        \centering
        \includegraphics[width=\textwidth]{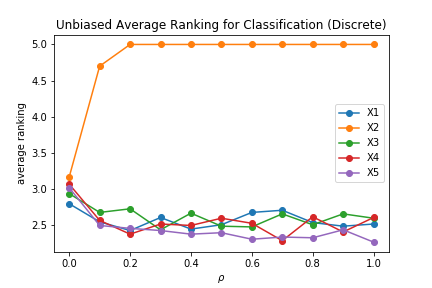}
        \caption{Classification}
        \label{fig:uc_rank_dis}
    \end{subfigure}
    \hfill
    \begin{subfigure}[b]{0.45\textwidth}
        \centering
        \includegraphics[width=\textwidth]{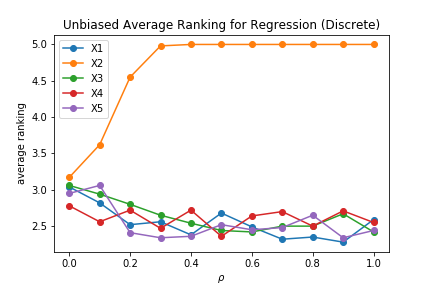}
        \caption{Regression}
        \label{fig:ur_rank_dis}
    \end{subfigure}
    \hfill
    
    \caption{Unbiased feature importance ranking across different signal strengths averaged over 100 repetitions, where we treat categorical features as ordered discrete values. 100 trees are built in the forest and maximum depth of each tree is set to 5. Each tree is trained using bootstrap samples and \textit{out-of-bag} samples are used as test set.}
    \label{fig:urank_dis}
    
\end{figure}

\end{document}